\documentclass[letterpaper, 10 pt, conference]{ieeeconf}  

\IEEEoverridecommandlockouts                              

\usepackage{amsmath}
\usepackage{amssymb}
\usepackage{dsfont}
\usepackage{color}
\usepackage{subcaption}
\usepackage{graphicx}
\usepackage{theorem}
\usepackage{titlesec}
\usepackage{url}
\usepackage[
colorlinks = true,
urlcolor  = magenta,
]{hyperref}
\usepackage{booktabs}
\usepackage{multirow}
\usepackage{setspace}
\usepackage{bm}
\usepackage[dvipsnames]{xcolor}
\usepackage{algorithm}
\usepackage{algpseudocode}

\usepackage{bbm}
\usepackage{mathtools}

\definecolor{myblue}{RGB}{0, 78, 158}

\newtheorem{proposition}{Proposition}

\newcommand{\partitle}[1]{\noindent{\textbf{#1.}}}
\newcommand{\bsl}[1]{\textbf{#1:}}

\newcommand{\interval}{interfering interval}
\newcommand{\intervals}{interfering intervals}

\title{\LARGE \bf
Reliable and Efficient Multi-Agent Coordination via Graph Neural Network Variational Autoencoders
}

\author{Yue Meng$^{1}$, Nathalie Majcherczyk$^{2}$, Wenliang Liu$^{2}$, Scott Kiesel$^{2}$, Chuchu Fan$^{1}$ and Federico Pecora$^{2}$
\thanks{*This research was done during Yue's internship at Amazon. Project page: \url{https://mengyuest.github.io/gnn-vae-coord/}}
\thanks{$^{1}$Yue Meng and Chuchu Fan are with the Massachusetts Institute
of Technology, Cambridge, MA 02139 USA. Email: {\tt\small \{mengyue,chuchu\}@mit.edu}}%
\thanks{$^{2}$Authors are with Amazon Robotics, North Reading, MA USA. Email: {\tt\small 
\{majcherc,liuwll,skkiesel,fpecora\}@amazon.com}}%
}

\begin{document}

\maketitle
\thispagestyle{empty}
\pagestyle{empty}

\begin{abstract}

Multi-agent coordination is crucial for reliable multi-robot navigation in shared spaces such as automated warehouses. In regions of dense robot traffic, local coordination methods may fail to find a deadlock-free solution. In these scenarios, it is appropriate to let a central unit generate a global schedule that decides the passing order of robots.
However, the runtime of such centralized coordination methods increases significantly with the problem scale. 
In this paper, we propose to leverage Graph Neural Network Variational Autoencoders (GNN-VAE) to solve the multi-agent coordination problem at scale faster than through centralized optimization. 
We formulate the coordination problem as a graph problem and collect ground truth data using a Mixed-Integer Linear Program (MILP) solver.
During training, our learning framework encodes good quality solutions of the graph problem into a latent space. 
At inference time, solution samples are decoded from the sampled latent variables, and the lowest-cost sample is selected for coordination.
By construction, our GNN-VAE framework returns solutions that always respect the constraints of the considered coordination problem. Numerical results show that our approach trained on small-scale problems can achieve high-quality solutions even for large-scale problems with 250 robots, being much faster than other baselines.

\end{abstract}
\section{INTRODUCTION}



Multi-agent coordination is essential to ensure that a fleet of robots can navigate shared spaces, such as warehouse floors~\cite{wurman2008coordinating} and public roads~\cite{adler2002cooperative}. Effective coordination avoids collisions, reduces delays, and optimizes resource usage.
Coordination between robots can either be achieved \textit{implicitly} by each robot acting to avoid conflicts based on its local information, or \textit{explicitly} via distributed or centralized decision-making. The former category of methods implies a pre-determined mutual understanding between robots (e.g. a set of rules or reciprocal policies). Their myopic nature is ill-suited for solving complex scenarios with many agents. The latter methods can plan ahead to optimize fleet operation, allowing robots to achieve common goals safely and efficiently in challenging settings. 



However, existing explicit coordination methods face a fundamental trade-off between optimality and computational tractability, particularly as the number of robots increases or the task objectives become more complex. While heuristic-based methods~\cite{gere1966heuristics} and sampling-based methods~\cite{omidvar2010comparative} are fast in computation, they often struggle to provide high-quality solutions for large graphs and require carefully crafted designs tailored to specific objectives. On the other hand, exact methods such as optimization-based approaches~\cite{kyriakidis2012milp} and search-based methods~\cite{li2021eecbs} can deliver better quality results, but their exponential complexity makes them impractical for large-scale problems.





In light of the recent advances in deep generative models~\cite{kipf2016variational}, we leverage Graph Neural Networks (GNN) and Variational Autoencoders (VAE) to learn the distribution of the high-quality solutions for multi-agent coordination problems. This approach offers several key advantages: \textit{(i)} GNN are well-suited for embedding the inherent graphical structure of multi-agent coordination problems, enabling them to capture complex interactions among robots. \textit{(ii)} VAE incorporate uncertainties in the problem, opening the possibility to generate multiple candidate solutions. \textit{(iii)} Neural Networks are efficient in evaluation, leveraging GPU parallel computation for faster performance, and \textit{(iv)} deep generative models based on graphs can generalize effectively to larger-scale problems. 

In this paper, we propose a GNN-VAE based framework to achieve reliable and efficient multi-agent coordination. Framing the multi-agent coordination problem as an optimization problem on a graph, we collect optimal solutions using a Mixed-Integer Linear Program (MILP) solver. During training, the GNN-VAE encodes these solutions into a latent space. At the inference stage, latent embeddings are sampled from the latent space and are further decoded to the solution samples, with the solution sample having the lowest cost selected for the coordination problem. Rather than predicting pure solution labels, GNN-VAE learns node ranks and edge modes in a semi-supervised manner to construct the solution, ensuring the prediction satisfies formal constraints of the coordination problem.


Our contributions can be summarized as follows: 
\begin{enumerate}
    \item We propose a novel learning framework that utilizes GNN-VAEs to tackle the particular application of multi-agent navigation in shared space, leveraging the generative nature of the model to sample from the set of feasible problem solutions.
    \item We propose a two-branch learning framework that guarantees, by construction, the satisfaction of two types of constraints of the coordination problem when inferring solutions.
    \item We perform an extensive evaluation of our approach, benchmarking it against strong baselines for problems involving up to 250 robots.
\end{enumerate}


\section{RELATED WORK}
\label{sec:related}


This paper considers centralized, explicit coordination problems, which belong to resource-constrained project scheduling problems (RCPSP)~\cite{pritsker1969multiproject} known to be NP-hard~\cite{blazewicz1983scheduling}. Related work can be divided into heuristic-based methods~\cite{de2007heuristic},  optimization-based methods~\cite{kyriakidis2012milp}, search-based methods~\cite{li2021eecbs} and sampling-based methods~\cite{xue2022multi}. An extensive comparison in ~\cite{rudy2022multiple} shows that meta-heuristic methods such as Tabu search outperforms other algorithms, and optimization-based methods work well on small-scale problems. Our approach does not require handcrafted heuristics designs, nor does it require time-consuming search or optimization processes. Instead, our method is akin to the sampling-based methods as it learns the underlying solution distribution from the demonstrated data, enabling it to scale and generalize to large-scale unseen scenarios.

Recent advances in neural networks have introduced data-driven approaches for multi-agent systems~\cite{garg2024learning}. Graph Neural Networks (GNN)~\cite{kipf2016semi,zhou2020graph,brody2021attentive} show a significant advantage in representing complex interactions between robots and generalize well to new scenarios~\cite{li2022online,yu2023learning,zhang2024gcbf+}. Deep generative models such as Variational Autoencoders (VAE)~\cite{kingma2013auto}, Generative Adversarial Networks (GAN)~\cite{goodfellow2020generative} and Diffusion models~\cite{ho2020denoising} have shown great success in learning from demonstrated data~\cite{song2018multi,ivanovic2019trajectron,jiang2023motiondiffuser}. Inspired by these contributions, we propose to utilize Graph Variational Autoencoders~\cite{kipf2016variational} to learn solution distribution for explicit coordination problems. The closest paper to ours is~\cite{wang2020learning}, which uses GNN to solve multi-robot coordination tasks with two to five robots. We consider tasks with diverse dependencies among robots with density constraints, and with the novel structural design, our method is guaranteed to generate feasible solutions and can scale up to 250 robots.


\section{PRELIMINARIES}
\label{sec:prelim}
\partitle{Robot configurations and paths}\footnote{We mainly follow the coordination graph formulation in~\cite{mannucci2021provably} and~\cite{rudy2022multiple}.}
Consider $N$ robots navigating in a shared 2D environment $\mathcal{W}\subseteq\mathbb{R}^2$ with static obstacles $\mathcal{O}\subseteq \mathcal{W}$. The $i$-th robot's configuration space is $Q_i\subseteq \text{SE}(2)=\mathbb{R}^2 \times \mathbb{S}^1$ where a configuration consists of the 2D position and heading angle. We define obstacle-free configurations for the $i$-th robot as $Q_i^{\text{free}}=\{q_i\in Q_i: R_i(q_i) \cap \mathcal{O}=\phi\}$ with $R_i:Q_i\to 2^{\mathbb{R}^2}$ indicating the robot's occupancy in the environment. Given start, goal configurations $q_i^s, q_i^g\in Q_i^{free}$, a path is a function $p_i: [0,1]\to Q_i^{free}$ that satisfies $p_i(0)=q_i^s$,  $p_i(1)=q_i^g$, and other kinematic constraints. 




\begin{figure*}[!htbp]
    \centering
    \includegraphics[width=0.85\textwidth]{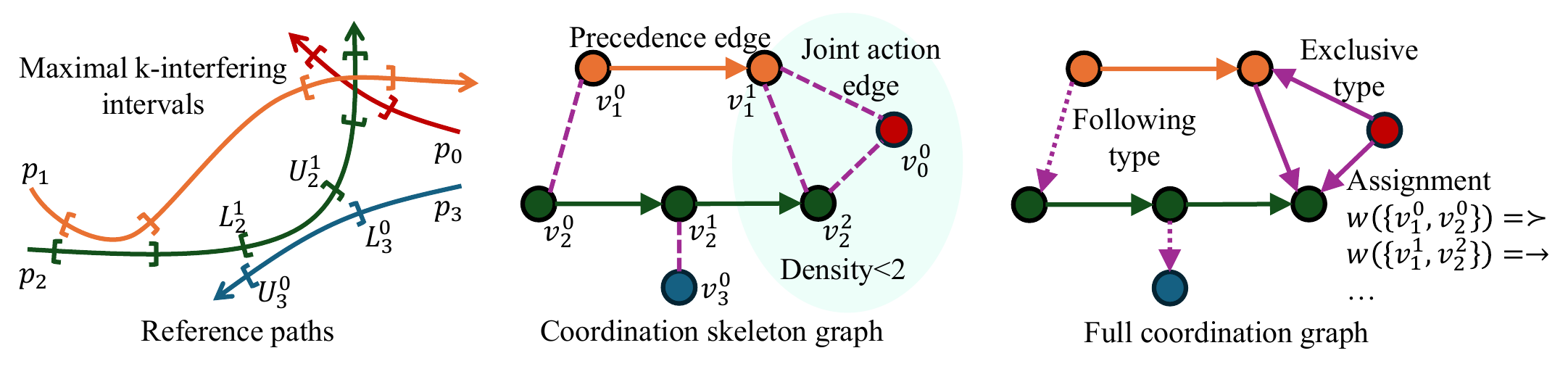}
    \caption{Illustration for the coordination graph.}
    \label{fig:graph-illustration}
\end{figure*}

\partitle{Interfering intervals}
Given an interference relation $\xi:Q_i^{free} \times Q_j^{free}\to [0,1]$, with 1 indicating the collision with two robots and 0 being collision-free, a pair of \textit{k-interfering intervals} $([l_i, u_i], [l_j, u_j]) \subseteq ([0,1]\times [0,1])$ for paths $p_i, p_j$ is defined as (here $0<k<1$):
\begin{equation}
\begin{aligned}
    &\left\{\forall \sigma_i\in[l_i,u_i], \exists \sigma_j\in[l_j,u_j], \text{ s.t. } \xi(p_i(\sigma_i), p_j(\sigma_j)) \geq k \right\}\land \\
    & \left\{\forall \sigma_j\in[l_j,u_j], \exists \sigma_i\in[l_i,u_i], \text{ s.t. } \xi(p_i(\sigma_i), p_j(\sigma_j)) \geq k\right\}
\end{aligned}
\end{equation}
and we will refer to this pair of intervals as an \textit{interfering section} between two robots. 
These interfering intervals are \textit{maximal k-interfering intervals} if they cannot be further expanded while satisfying the conditions above. For brevity, we will refer to them as \intervals{} for the rest of the paper.
For a pair of paths $p_i, p_j$, denote the set of all \intervals{} as $\Xi(p_i, p_j, k)$. 

\partitle{Coordination graphs} Given all the \intervals{} for the $N$ robots $\cup_{i=1}^{N-1} \cup_{j=i+1}^N \Xi({p}_i, {p}_j, k)$, we construct a mixed graph $G=(V, P, A)$ with the vertices set $V$, the directed edge set $P$ and the undirected edge set $A$. The node $v_i^p\in V$ is associated with the  $p$\textsuperscript{th} merged \interval{} \footnote{If \intervals{} from different pairs overlap on a robot path, we merge them but keep the intervals' entering/exit time per interfering pair.} for the $i$\textsuperscript{th} robot. $P$ denotes all the ``precedence" edges: $(v_i^p,v_i^{p+1})$ and $A$ denotes all the unordered ``joint action" edges $\{v_i^p, v_j^q\}$ for each pair of \intervals{} $([l_i^p, u_i^p], [l_j^q, u_j^q])$ with $i\neq j$. The graph $G$ is called  \textit{coordination skeleton graph}. A \textit{(full) coordination graph} is a coordination skeleton graph with each joint action edge assigned to a value that decides the passing order and the passing pattern for the robots at the interfering section. For each pair $\{v_i^p, v_j^q\}\in A$, the joint action values are $\mathcal{D}=\{\rightarrow, \leftarrow, \succ, \prec \}$, where: 
\begin{itemize}
    \item \textit{Exclusive}: $v_i^p\rightarrow v_j^q$ imposes the $j$-th robot must wait to navigate beyond $l_j^q$ until the $i$-th robot has reached $u_i^p$.
    \item \textit{Following}: $v_i^p\succ v_j^q$ imposes the $j$-th robot must wait to navigate beyond $l_j^q$ until the $i$-th robot has reached $l_i^p$.
\end{itemize}

\partitle{Problem constraints} Denote $w: A \to \mathcal{D}$ the function to assign joint action edges with values, and the \textit{assignment} for a graph $w_G=\{w(\{v_i^p,v_j^q\}) \,|\, \{v_i^p, v_j^q\}\in A\}$. We have the constraints: (1) The directed graph $G_{w_G}$ induced by the skeleton graph $G$ and the assignment $w_G$ is acyclic (no cycles in the graph), and (2) the number of ``following"-type edges is restricted. The former (acyclic constraint) is to avoid deadlocks caused by a ``circular waiting" among the robots, and the latter (density constraint) limits the maximum number of robots allowed to pass the interfering section simultaneously.
The density constraint is enforced on maximal cliques~\footnote{A clique is a subset of vertices where every pair is connected by an edge, while a maximal clique is a clique that cannot be extended by including any additional adjacent vertex.} on the subgraph $G'=(V, A)$: for each maximal clique $K\in\mathcal{K}(G')$ with a density constraint $\rho_{K}$, the number of ``following"-type edges should be no more than $h_K=\frac{(\rho_{K}+1)\rho_{K}}{2}-1$.\footnote{The minimum number of 'following' edges needed for $k$ robots to pass through an interfering region at once is $\frac{k(k-1)}{2}$. Hence, the maximum allowed for density $\rho_K$ is $\frac{(\rho_{K}+1)\rho_{K}}{2}-1$.} An illustration is shown in Fig.~\ref{fig:graph-illustration}.

\partitle{Travel time under assignment} We consider the updated robots' travel time as a main objective to measure the assignment quality.
Each \interval{} $[l_i^p, u_i^p]$ is associated with an \textit{expected travel time interval} $[L_i^p, U_i^p]$ indicating the scheduled time for the i\textsuperscript{th} robot to enter $l_i^p$ and to exit $u_i^p$ if there is no interference. Given an assignment $w_G$, the \textit{updated travel time intervals} $[\tilde{L}_i^p, \tilde{U}_i^p], [\tilde{L}_j^q, \tilde{U}_j^q]$ for robots $i, j$ at the interfering section should satisfy:
\begin{equation}
\begin{aligned}
&\tilde{L}_i^p\geq \tilde{U}_j^q, \text{if } w(\{v_i^p,v_j^q\})=\leftarrow; 
\tilde{L}_j^q\geq \tilde{U}_i^p, \text{if } w(\{v_i^p,v_j^q\})=\rightarrow\\
&\tilde{L}_i^p\geq \tilde{L}_j^q, \text{if } w(\{v_i^p,v_j^q\})=\prec;
\tilde{L}_j^q\geq \tilde{L}_i^p, \text{if } w(\{v_i^p,v_j^q\})=\succ.
\end{aligned}
\label{eq:delay-constraints-1}
\end{equation}

Globally, the updated travel time should satisfy monotone-increasing delay constraint along the path. 
Denote all the distinct expected time for the robot $i$ to enter and exit its \intervals{} sorted as $0\leq T^{(1)}_i\leq T^{(2)}_i \leq ...\leq T^{(C_i)}_i$. The updated travel time $\tilde{T}_i^{(1)}, \tilde{T}_i^{(2)},...,\tilde{T}_i^{(C_i)}$ should satisfy:
\begin{equation}
    0\leq \tilde{T}^{(1)}_i -T^{(1)}_i\leq \tilde{T}^{(2)}_i-T^{(2)}_i\leq ... \leq \tilde{T}^{(C_i)}_i-T^{(C_i)}_i.
    \label{eq:delay-constraints-2}
\end{equation}

The minimum values to satisfy the constraints in \eqref{eq:delay-constraints-1} and \eqref{eq:delay-constraints-2} for all the robots form the updated travel time. Here $\tilde{T}^{(C_i)}_i$ denotes the updated task finishing time of a robot, and $D^{(l)}_i=\tilde{T}^{(l)}_{i}-T^{(l)}_{i}$ denotes the delay at the $l$\textsuperscript{th} stage.

\section{Problem Formulation}
Given a coordination skeleton graph $G=(V,P,A)$, a set of density constraints $\mathcal{K}(G')$ on its subgraph $G'=(V,A)$, and an expected travel time $\{\{T^{(p)}_i\}_{p=1}^{C_i}\}_{i=1}^N$ for a multi-agent coordination problem defined in Sec.~\ref{sec:prelim}, we aim to find the optimal assignment $w_G^*\in\mathcal{W}$ that minimizes the cost function $f_{G,T}:\mathcal{W}\to \mathbb{R}$ while satisfying the acyclic and density constraints:
\begin{equation}
    \begin{aligned}
    \mathop{\text{Min}}\limits_{w_G\in \mathcal{W}} \quad & f_{G,T}(w_G) \\
    \text{s.t.} \quad & G_{w_G} \text{ is a directed acyclic graph (DAG).} \\
   \quad &\sum\limits_{v, v'\in K, v\neq v'} \mathds{1}_{(w(\{v,v'\})\in\{\succ, \prec\})} \leq h_K, \forall K \in \mathcal{K}(G') 
\end{aligned}
\label{eq:formulation}
\end{equation}
where $\mathds{1}_{(w(\{v,v'\})\in\{\succ, \prec\})}$ is 1 if the assignment for the edge $\{v, v'\}$ is ``following"-type, otherwise 0. Different cost functions will be explained in the following section.
\begin{figure*}[!htbp]
    \centering
    \includegraphics[width=0.9\textwidth]{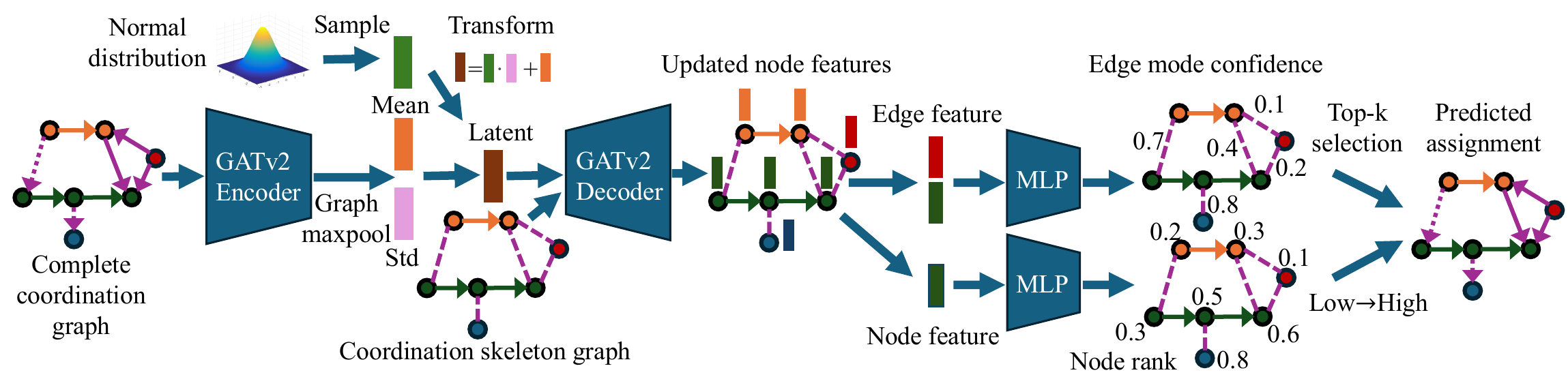}
    \caption{Learning framework: The GNN-VAE first encodes the assignment via graph convolution and graph max pooling to a latent embedding. The sampled latent code is sent to the decoder and the two-branch MLP to get the predicted assignment.}
    \label{fig:arch}
\end{figure*}

\section{TECHNICAL APPROACH}
\label{sec:tech}
\subsection{MILP formulation for assignment optimization}
Given the problem defined in Eq.~\eqref{eq:formulation}, we can find the optimal solution considering the following MILP. 
For every joint action edge $\{v_i^p,v_j^q\}\in \mathcal{A}$, we denote the binary decision variables $y_{ij}^{pq}$ to indicate whether an edge is pointing from $v_i^p$ to $v_j^q$ ($y_{ij}^{pq}=0$) or from $v_j^q$ to $v_i^p$ ($y_{ij}^{pq}=1$), and denote the binary decision variables $z_{ij}^{pq}$ to indicate whether an edge is an exclusive type ($z_{ij}^{pq}=0$) or a following type ($z_{ij}^{pq}=1$). Using the big-M method~\cite{hillier2015introduction} with $M$ a big positive number, the MILP formulation is:
\begin{equation}
    \begin{aligned}
         \min & \mathrlap{\,\, f(\{y_{ij}^{pq},z_{ij}^{pq}\}_{\{v_i^p,v_j^q\}\in {A}},\{L_i^p,U_i^p,\tilde{L}_i^p,\tilde{U}_i^p\}_{{v_i^p\in V}})}\\
         \text{s.t. } & M y_{ij}^{pq}  + M z_{ij}^{pq} + \tilde{L}_j^q     \geq  \tilde{U}_i^p , &\,\forall (v_i^p,v_j^q)\in A\\
                      & M (1-y_{ij}^{pq})  + M z_{ij}^{pq} + \tilde{L}_i^p \geq  \tilde{U}_j^q , &\,\forall \{v_i^p,v_j^q\}\in A\\
                      & M y_{ij}^{pq} + \tilde{L}_j^q   \geq \tilde{L}_i^p , &\,\forall \{v_i^p,v_j^q\}\in A\\
                      & M (1-y_{ij}^{pq}) + \tilde{L}_i^p  \geq  \tilde{L}_j^q , &\,\forall \{v_i^p,v_j^q\}\in A\\
                      & \sum\limits_{\{v_i^p,v_j^q\}\in K} z_{ij}^{pq} \leq h_K, &\, \forall K \in \mathcal{K}(G') \\
                      & \mathrlap{0\leq \tilde{T}_i^{(1)}-T_i^{(1)}\leq \tilde{T}_i^{(2)}-T_i^{(2)} \leq\cdots \leq \tilde{T}_i^{(C_i)}-T_i^{(C_i)}}\\
                      &&  \forall i \in \{1,2,...,N\}
    \end{aligned}
\end{equation}
where $\{T_i^{(l)}\}_{l=1}^{C_i}$ are the sorted distinct expected enter/exit time in the ascending order for the $i$\textsuperscript{th} robot.
The first four constraints are ``exclusive" and ``following" constraints. The next one is for density constraints. The last one indicates the monotone increasing delay for the robots. For the objective function, we consider the average completion time, the maximum completion time, the synchronized completion time and the average interference delay, defined as follows:
\begin{equation}
    \begin{aligned}
        & t_{avg}=\frac{1}{N}\sum\limits_{i=1}^N \tilde{T}^{(C_i)}_{i}, \quad t_{max}=\max\limits_{i=1,...,N} \tilde{T}^{(C_i)}_{i}\\
        & t_{sync}=t_{avg}+\frac{1}{N}\sum\limits_{i=1}^N |\tilde{T}^{(C_i)}_i-t_{avg}|, \\
        & t_{delay}=\frac{1}{N}\sum\limits_{i=1}^{N}\frac{1}{C_i}\sum\limits_{p=1}^{C_i}(\tilde{T}^{(p)}_{i}-T^{(p)}_{i}).
    \end{aligned}
\end{equation}

For each objective function and graph, we collect top $L$-optimal assignments using a MILP solver, which will be used to train our GNN-VAE model.

\subsection{Assignment prediction using GNN-VAE}

\partitle{Graph data encoding} The GNN-VAE's input has the same number of nodes as in the coordination graph, 
where we assign directed edge for precedence edges and bidirected edges for joint action edges.
The node feature for $v_i^p$ is $(L_i^p, U_i^p, \rho_i^p)$ which are the left and right expected travel time at the $p$\textsuperscript{th} merged interfere section and the density constraint. The edge feature for $(v_i^p,v_j^q)\in A$ on the completed graph $G$ is $(L_{ij}^{pq}, U_{ij}^{pq}, w_{ij}^{pq})$ with the expected enter/exit time for the \interval{} of $v_i^p$ when considering the interference with $v_j^q$, and $w_{ij}$ indicating the joint action type. The edge feature for the skeleton graph $\tilde{G}$ is $(L_{ij}, U_{ij}, 0)$.



\partitle{GNN-VAE learning framework}
In training, the assignment graph is sent to the graph encoder with global max-pooling to derive the latent embedding, which is used to reconstruct graph assignments. 
In testing, the embedding is directly sampled from a standard normal distribution.
In the decoding process, we concatenate the embedding node-wise on the skeleton graph and conduct message propagation. Here we use graph attention layer (GATv2) proposed in~\cite{brody2021attentive} for the encoder and the decoder. The resulting fused features are then utilized to generate the assignments.



\partitle{Violation-free assignment generation}
Our GNN-VAE predicts the node ranks and edge types to generate assignments that are guaranteed by design to satisfy the acyclic and density constraints. The fused feature vector for each node $v_i$ is sent into a multi-layer perceptron (MLP) to predict the node bid $b_i>0$. The nodes' ranks are computed from the bids under the following graph operation:
\begin{equation}
    \tilde{r}_i = b_i + \sum\limits_{j\in \mathcal{A}(i)} b_j
\end{equation}
where $\mathcal{A}(i)$ denotes all the ancestors for the node $v_i$ on the coordination graph with only precedence edges $G'=(V, P)$. 
The joint action edges direction then are determined by pointing from the lower-ranked nodes to the higher-ranked nodes. A variation of hinge loss is used to ensure the learned ranks consistent with the ground truth assignments:
\begin{equation}
    \mathcal{L}_{bar}=\sum_{\{v_i,v_j\}\in A}\left[\sigma_{+}(\tilde{r}_i-\tilde{r}_j)\mathds{1}_{ij} + \sigma_{+}(\tilde{r}_j-\tilde{r}_i) \mathds{1}_{ji}\right]
\end{equation}
where $\mathds{1}_{ij}=1$ if the edge is from $v_i$ to $v_j$ in the assignment and 0 otherwise, and $\sigma_+(x)=\max(x+\gamma,0)$ with a bloating factor $\gamma>0$ for numerical stability.

To determine if an undirected edge $\{v_i, v_j\}$ is ``exclusive" or ``following", we use an MLP with input the fused node features from $v_i, v_j$ to predict the edge type with the binary cross-entropy loss:
\begin{equation}
    \mathcal{L}_{bce}=\sum\limits_{\{v_i,v_j\}\in A}\left[y_{ij}\log(\hat{p}_{ij}) + (1-y_{ij})\log(1-\hat{p}_{ij})\right]
\end{equation}
here $\hat{p}_{ij}$ is the estimated probability for the edge $\{v_i,v_j\}$ being  ``following"-type, and $y_{ij}$ is the binary ground truth label (with 1 being the ``following"-type). Upon assignment generation, for each maximal clique on the graph, we sort edges based on $p_{ij}$ and select the top-$h_k$ edges to be ``following"-type. This formulation guarantees that the assignment always adheres to the density limit.

Finally, we use a KL-divergence loss to regularize the latent space distribution to be similar to a standard normal distribution and
the final loss becomes:
\begin{equation}
    \mathcal{L} = \alpha_1\mathcal{L}_{bar} + \alpha_2\mathcal{L}_{bce} + \alpha_3\mathcal{L}_{kl}
\end{equation}
where $\alpha_1,\alpha_2,\alpha_3$ weighs the balance between loss terms.

\partitle{Remarks} 
Our method is reliable and expressive: The generated assignments can always satisfy the acyclic and density constraints. Moreover, it can produce a corresponding assignment for any directed acyclic graph (DAG) with any distribution of the following-type edges.
The density constraints are met since the top-k selection mechanism allows at most $h_k$ robots into an \interval{} at once. Therefore, we just complete our proof of this statement for the acyclic constraints below.

\begin{figure*}[!htbp]
    \centering
    \begin{subfigure}[b]{0.243\textwidth}
        \centering
        \includegraphics[width=\textwidth]{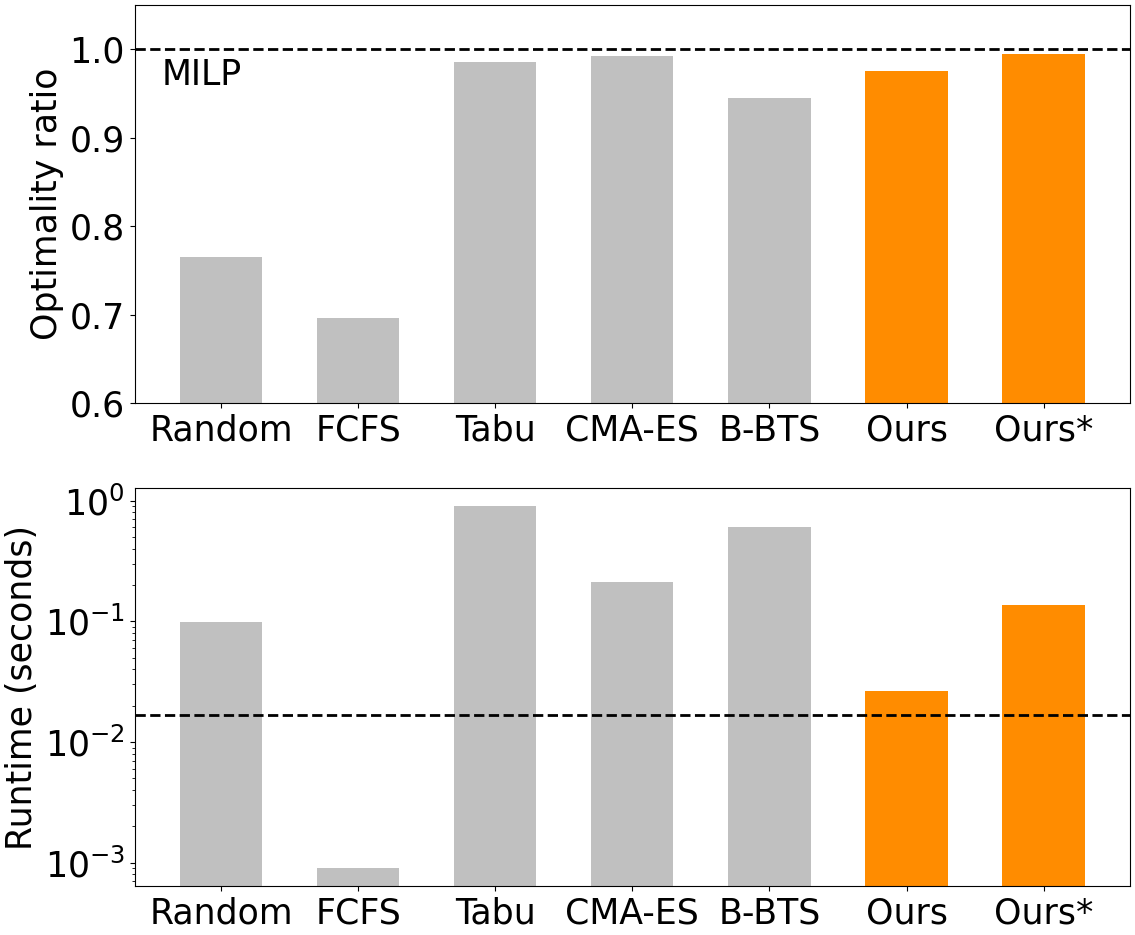}
        \caption{Average completion time}
        \label{fig:tavg}
    \end{subfigure}
    \begin{subfigure}[b]{0.243\textwidth}
        \centering
        \includegraphics[width=\textwidth]{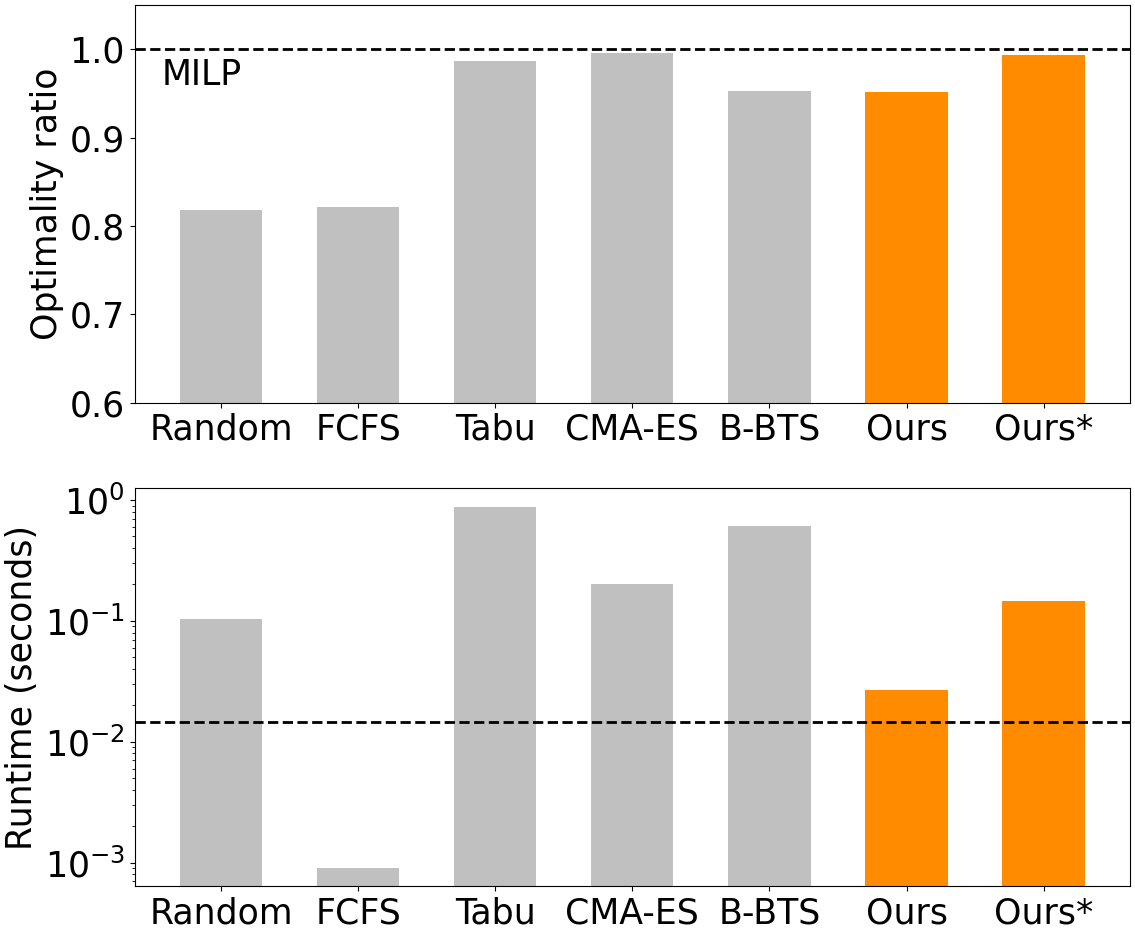}
        \caption{Maximum completion time}
        \label{fig:tmax}
    \end{subfigure}
    \begin{subfigure}[b]{0.243\textwidth}
        \centering
        \includegraphics[width=\textwidth]{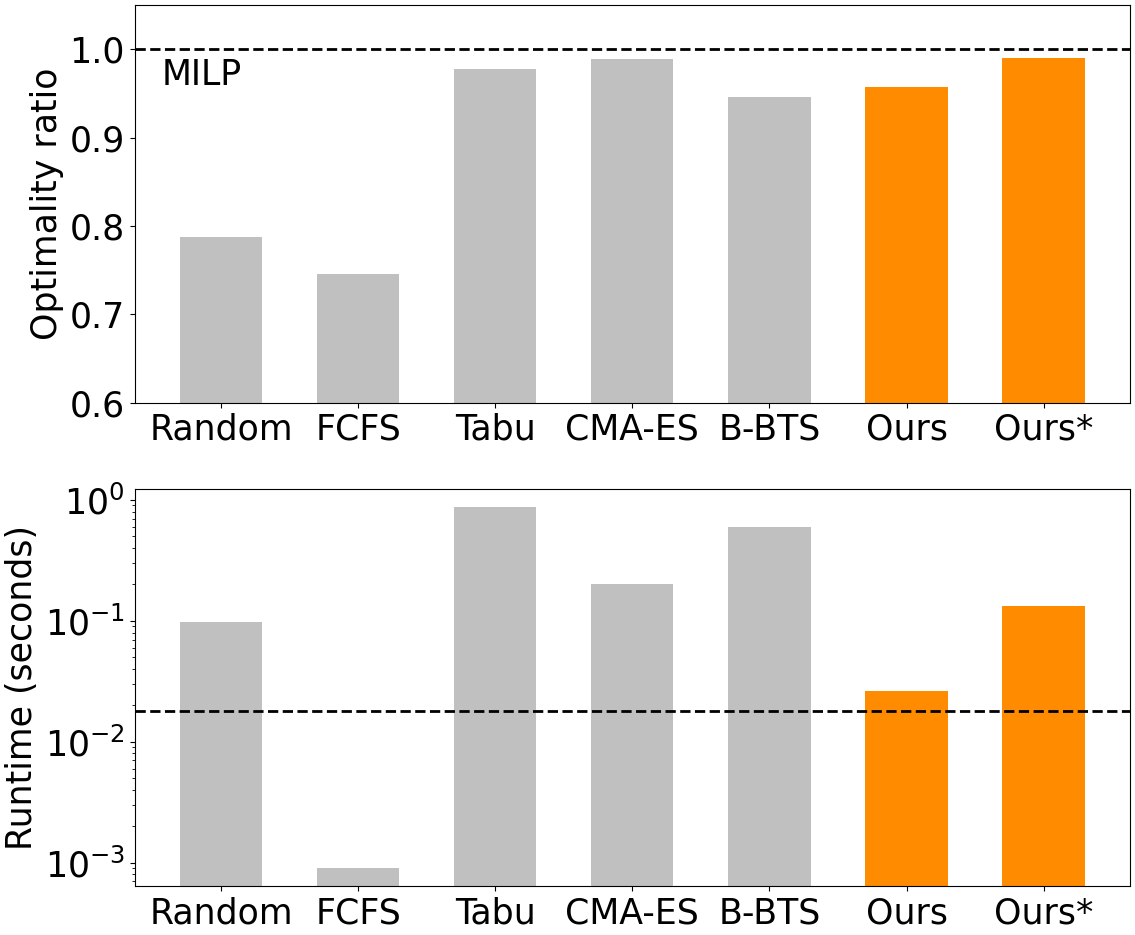}
        \caption{Synchronized completion time}
        \label{fig:tsync}
    \end{subfigure}
    \begin{subfigure}[b]{0.243\textwidth}
        \centering
        \includegraphics[width=\textwidth]{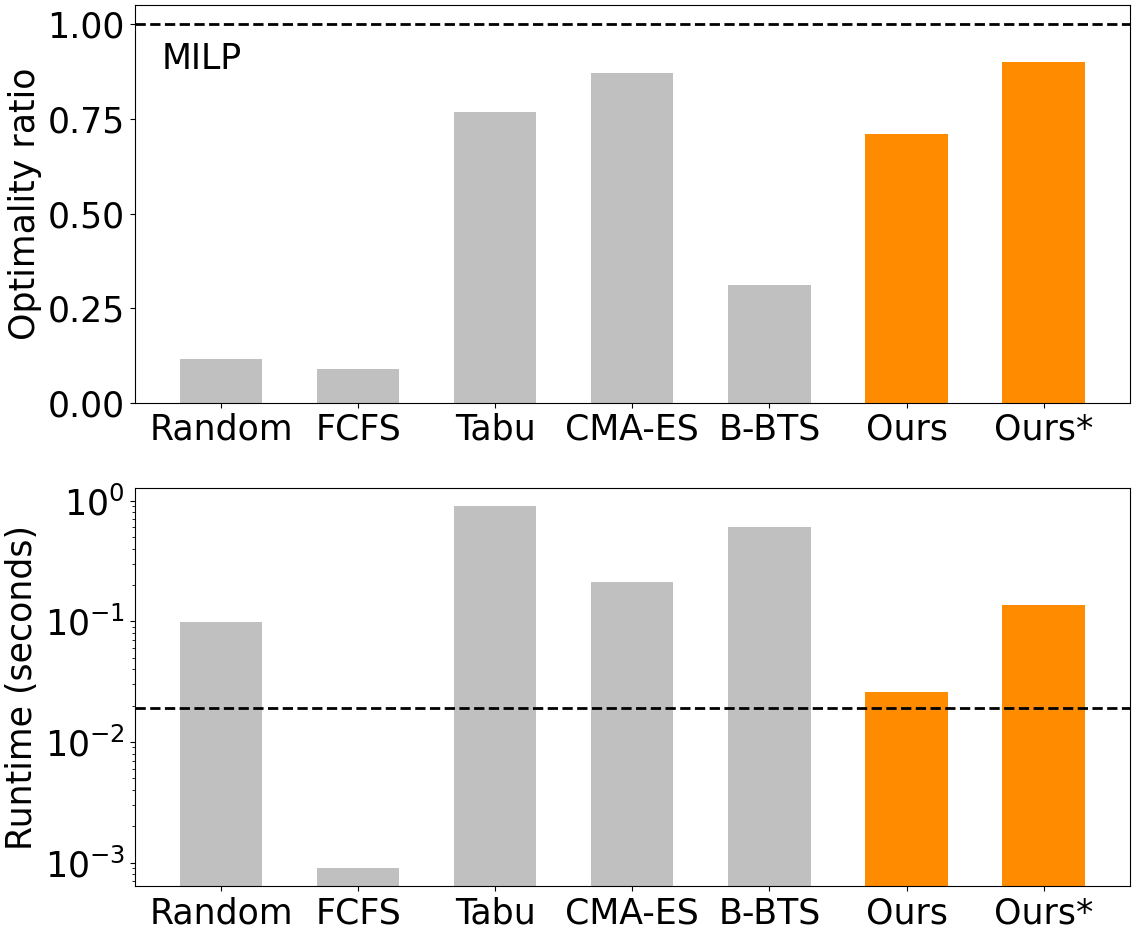}
        \caption{Average interfere delay}
        \label{fig:tdelay}
    \end{subfigure}
    \caption{Main comparisons for solution quality and computation runtime under different cost functions.}
    \label{fig:main-cmp}
\end{figure*}

\begin{proposition}
Given a mixed graph $G=(V,E,S)\in\mathcal{G}$ with disjointed directed edges $E$ and undirected edges $S$ and $G'=(V,E)\in\mathcal{G}_d$  a directed acyclic graph (DAG), denote the transformation from the mixed graph and node bids to a new directed graph as $\mathcal{T}:\mathcal{G}\times\mathcal{B}_V\to\mathcal{G}_d$.
The following properties hold: (1) $\forall b_V\in\mathcal{B}_V$,  $G_{new}=\mathcal{T}(G,r_V)$ is a DAG containing $E$. (2) $\forall E_{new}$, if $G_{new}=(V, E\cup E_{new})$ is a DAG, then $\exists b_V\in\mathcal{B}_V$ such that $\mathcal{T}(G, b_V)=G_{new}$.
\end{proposition}
\begin{proof}
(1) By~\cite{cormen2022introduction} (Section 22.4), the partial order on the set of nodes $V$ induces a DAG. 
Since the node ranks $\tilde{r}_V$ derived from the node bids $b_V$ is a valid partial order, we know the $G_{new}$ is DAG, and it remains to show that the newly induced graph contains edges in $E$, i.e., $\forall (u, v)\in E, \tilde{r}_u < \tilde{r}_v$. Since $u$ and all its ancestors are also the ancestors of $v$, and the node ranks are all positive, $\tilde{r}_v = b_v + \sum\limits_{i \in \mathcal{A}(v)} b_i  \geq b_v + b_u + \sum\limits_{j \in \mathcal{A}(u)} b_j \geq b_v + \tilde{r}_u > \tilde{r}_u$. Thus, $G_{new}=\mathcal{T}(G, b_V)$ is a DAG that contains $E$.

(2) We prove it by construction. 
Any DAG has at least one topological ordering, where, for any edge $(u,v)$ on the DAG, $u$ appears before $v$. We assign the node bids following the topological order for $G_{new}$ so that the ancestors' bids and ranks are available when determining the current node's bid. Denote the current node as $v$ and then $\forall(u, v)\in E$, any positive $b_v$ will ensure $\tilde{r}_v>\tilde{r}_u$ (proved above). To ensure $\forall (w, v)\in E_{new}$, $\tilde{r}_v>\tilde{r}_w$, we can assign $b_v=\max\{0, \max\limits_{j\in A_v}\tilde{r}_j - \sum\limits_{k\in A_v}b_k \} + \epsilon$ with $\epsilon>0$ and $A_v$ denotes the set of all the ancestors for the current node. Thus, we can get $\tilde{r}_v = b_v + \sum\limits_{k\in A_v}b_k > \max\limits_{j\in A_v}\tilde{r}_j + \epsilon> \tilde{r}_w$. We continue this process to get all the node bids $b_V$, which, by construction, generates $G_{new}=(V, E\cup E_{new})$.
\end{proof}

\section{EXPERIMENTS}
\label{sec:exp}

\partitle{Implementation details}  We randomly generate 10000 coordination problems with 2 to 8 robots and up to 14 interfering sections in each case. For each cost function, we use the Gurobi MILP solver~\cite{gurobi} to generate the top-10 optimal assignments and form the training and validation datasets. In our GNN-VAE learning framework, the encoder and decoder are GATv2 layers~\cite{brody2021attentive}, and the node/edge prediction heads are MLPs. Both GATv2 and MLP are implemented with 4 hidden layers, 256 units in each layer, and a ReLU activation is used for the intermediate layers. The learning pipeline is implemented in Pytorch Geometric~\cite{fey2019fast,paszke2019pytorch}. The training is conducted with an ADAM~\cite{kingma2014adam} optimizer, a learning rate $3\times 10^{-4}$ and a batch size 128. The coefficients are $\alpha_1=1.0, \alpha_2=1.0, \alpha_3=0.01, \gamma=0.1$. The training takes 1$\sim$ 2 hours on an NVidia A100 GPU. During evaluation, for each graph, we sample 100 assignments from the GNN-VAE decoder and select the one with the lowest cost.

\partitle{Baselines} We consider: (1) \bsl{Random} randomly generate node ranks and joint action edge types to form a valid assignment (2) \bsl{FCFS} first-come-first-serve to assign the joint action direction (if A has an earlier entering time than B, an edge points from A to B) and randomly generate edge types (3) \bsl{Tabu} a local search algorithm based on Tabu Search~\cite{glover1998tabu}, initialized with the solution from FCFS (4) \bsl{CMA-ES} Covariance Matrix Adaptation Evolution Strategy~\cite{hansen2003reducing}, (5) \bsl{B-BTS} budgeted backtrack search that finds the first 1000 feasible candidates and select the one with the lowest cost, and (6) \bsl{MILP} mixed-integer linear program (treated as the oracle since it generates optimal solutions).

\partitle{Metrics} (1) \bsl{Optimality ratio} the ratio of the oracle assignment cost to the predicted assignment cost, a number between (0,1] to measure the assignment quality (the closer to 1, the better quality of the assignment) (2) \bsl{Computation runtime} the average runtime to solve a problem.

\subsection{Main results on small-scale problems}
We train our GNN-VAE on four datasets with varied cost functions and evaluate on the validation set. As shown in Fig.~\ref{fig:main-cmp}, regarding the optimality ratio, \textbf{Ours} outperforms \textbf{Random}, \textbf{FCFS} and \textbf{B-BTS}, achieving a comparable performance to strong baselines such as \textbf{Tabu} and \textbf{CMA-ES} while being one magnitude faster than both approaches in the computation runtime. With a few CMA-ES refinement steps conducted based on our predicted assignment solution (denoted as \textbf{Ours*}), we achieve the closest-to-oracle (MILP) solution quality with a slight increase in the computation time. Our method demonstrates a consistent advantage across varied objective functions, with the most significant improvement over baselines observed on the ``Average interfere delay" cost. This result is intuitive, as the delay metric captures the absolute difference in robot progress at each interference section, and is therefore not affected by the total length of the progress. This shows our GNN-VAE can effectively learn to capture the optimal solution distribution and achieves a better trade-off between the solution quality and the inference runtime compared to other approaches.

\begin{figure}[!htbp]
    \centering
    \includegraphics[width=0.4\textwidth]{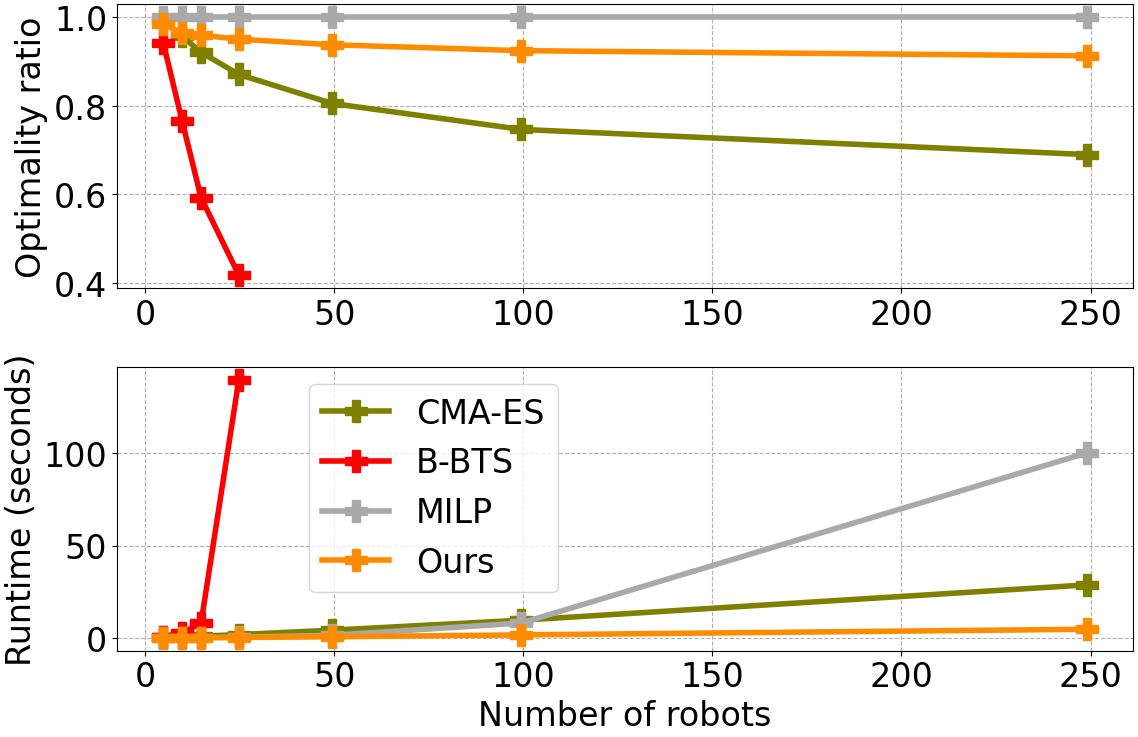}
    \caption{Performance over larger graphs.}
    \label{fig:scalability}
\end{figure}

\begin{figure}[!htbp]
    \centering
    \begin{subfigure}[b]{0.235\textwidth}
        \centering
        \includegraphics[width=\textwidth]{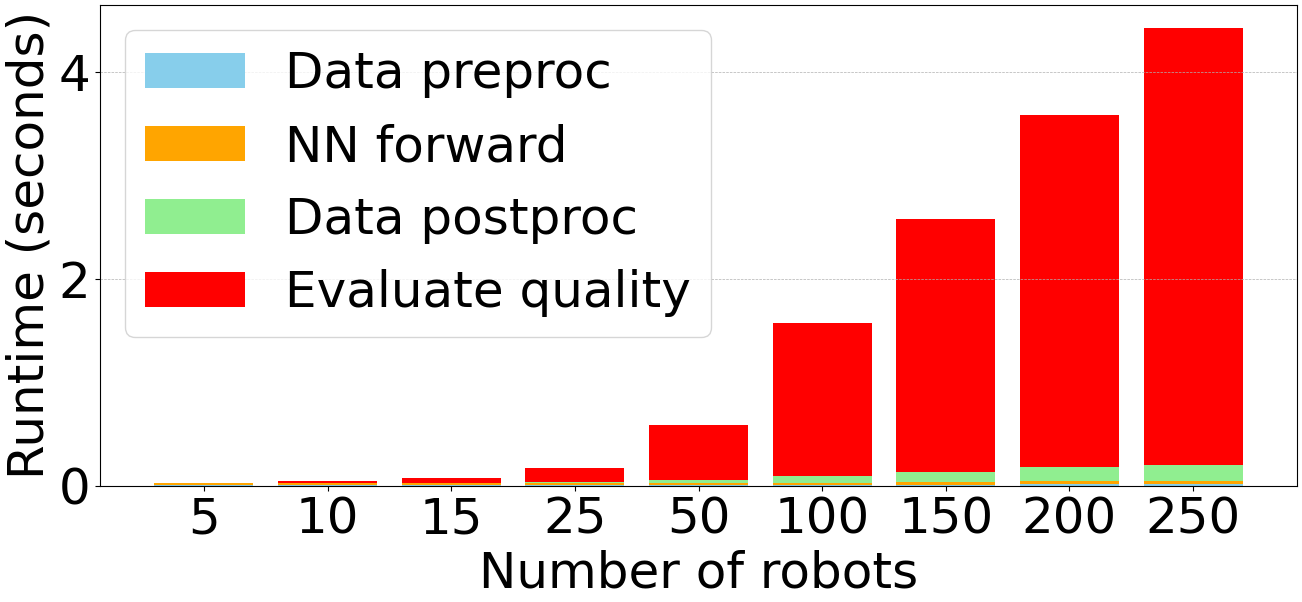}
        \caption{Absolute runtime (seconds)}
        \label{fig:runtime-sub1}
    \end{subfigure}
    \begin{subfigure}[b]{0.235\textwidth}
        \centering
        \includegraphics[width=\textwidth]{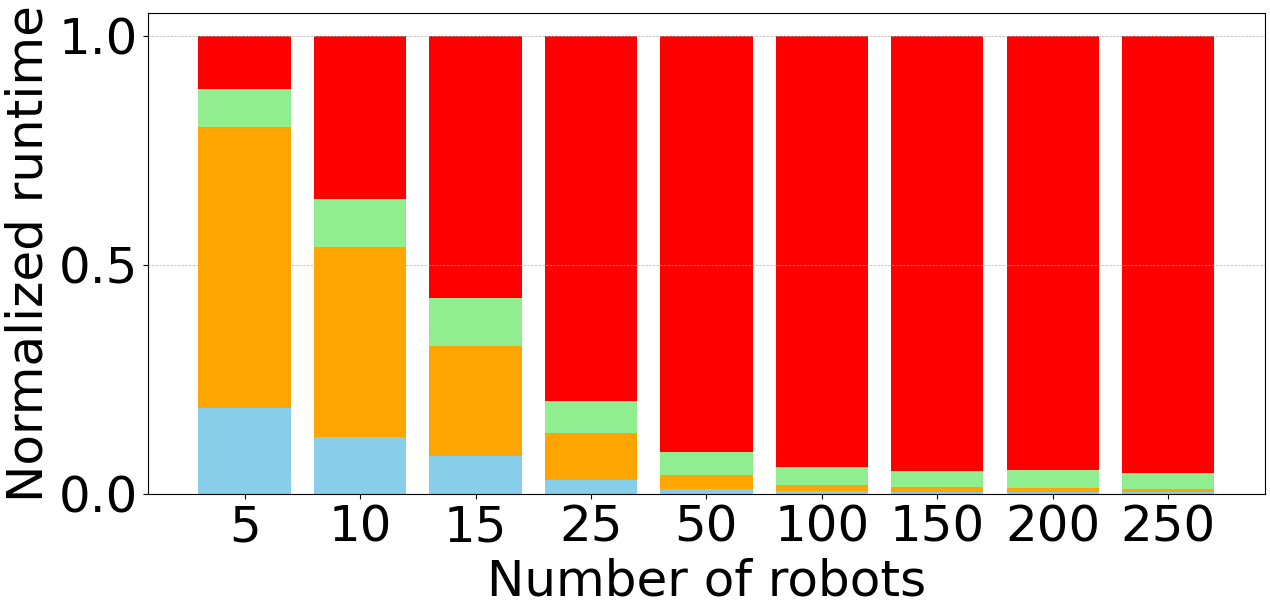}
        \caption{Normalized runtime}
        \label{fig:runtime-sub2}
    \end{subfigure}
    \caption{Runtime breakdown for GNN-VAE at inference stage.}
    \label{fig:runtime}
\end{figure}

\subsection{Generalizability and scalability to large-scale problems}

The main advantages of our GNN-VAE are that it can generalize well to larger graphs without retraining and that it can scale better than other search-based or optimization-based methods. We generate large-scale problems by (1) creating the coordination subgraphs following the procedure as creating small-scale problems and (2) randomly stitching subgraphs together by adding more interfering relations over vertices from different subgraphs. The original problem involves on average 5 robots, whereas the average numbers of robots on the new generated graphs range from 10 to 250. We select the GNN-VAE model pre-trained on the ``Average completion time" dataset and directly compare it with strong baselines \textbf{B-BTS}, \textbf{MILP} and \textbf{CMA-ES}\footnote{We did not compare with \textbf{TABU} because it is too time-consuming on the larger graphs, and we did not compare with \textbf{FCFS} or \textbf{Random} because they cannot produce quality solutions.} on the newly generated varied-size large-scale problems. In this stage, we do not conduct further CMA-ES refinement due to the time limit. As shown in Fig.~\ref{fig:scalability}, our approach can generate close-to-oracle assignments with the optimality ratio consistently over 0.9 while the optimality ratio curves for \textbf{B-BTS} and \textbf{CMA-ES} drop quickly as the number of robots is more than 20. This shows the great generalizability of our approach. Regarding the algorithm runtime, our approach can be 10 to 20 times faster than the baselines, and we can solve the coordination problem with 250 robots in less than 5 seconds on average. Fig.~\ref{fig:runtime} shows a runtime breakdown analysis where it reveals that the bottleneck is not from data processing or the neural network operations, but from measuring the assignment quality, which can be computed in a parallel fashion as they do not depend on each other. We believe this could further improve our runtime performance.

\subsection{Test on out-of-distribution data in simulation} 
We randomly generate disk-shaped and rectangular obstacles in a 2d environment and use a search-based path planner~\cite{likhachev2010search} to generate reference paths for the robots. Next, we create the coordination graph based on these reference paths and use our GNN-VAE to generate an assignment. Then we compute the updated travel time for robots at interfering sections. At every simulation step, if the time is before the updated travel time, the robot will wait on the reference path; otherwise, the robot will track the reference path. We conduct the simulation in PyBullet environment~\cite{coumans2016pybullet}. The screenshot for the simulation and the cost ratio are shown in Fig.~\ref{fig:demo}. We can see that our model pretrained on the small-scale synthetic graph dataset can generalize to out-of-distribution scenarios, providing close-to-oracle quality schedules for up-to-eight robots, and still better than \textbf{Random} for $<$10 robots.

\begin{figure}[!htbp]
    \centering
    \begin{subfigure}[b]{0.21\textwidth}
        \centering
        \includegraphics[width=\textwidth]{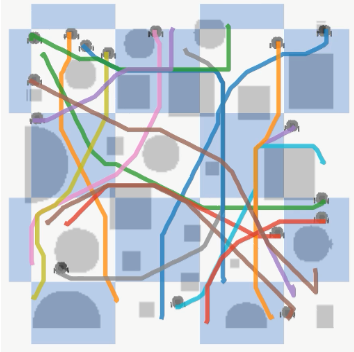}
        \caption{Simulation screenshot.}
        \label{fig:sub1-sim}
    \end{subfigure}
    \begin{subfigure}[b]{0.25\textwidth}
        \centering
        \includegraphics[width=\textwidth]{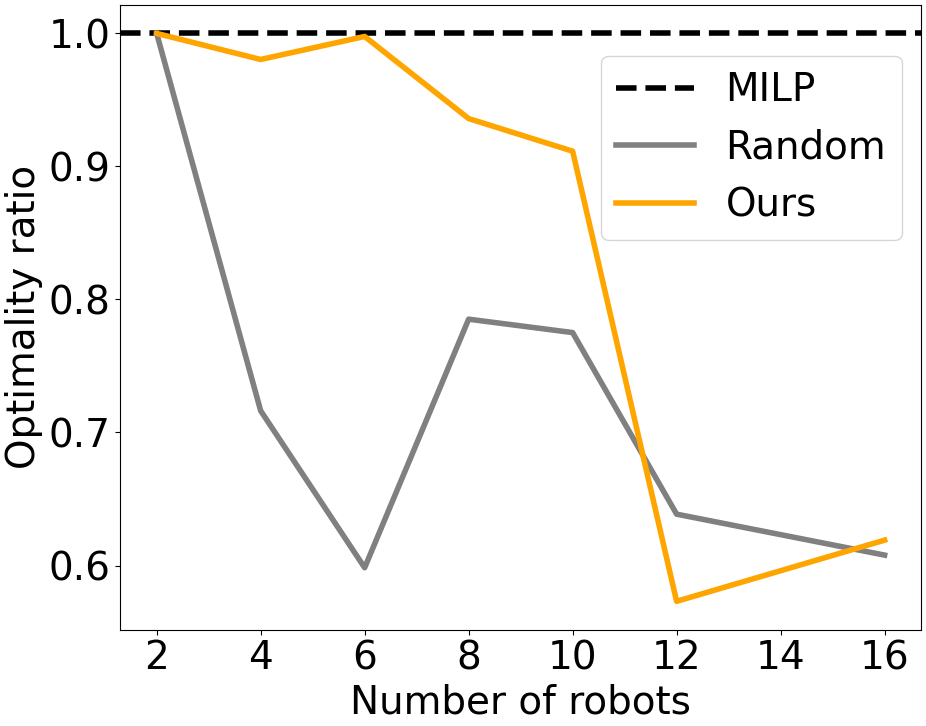}
        \caption{Cost ratio curve under varied number of robots in simulation.}
        \label{fig:sub2-cost}
    \end{subfigure}
    \caption{Out-of-distribution test in simulation environments.}
    \label{fig:demo}
\end{figure}

\begin{figure}[!htbp]
    \centering
    \includegraphics[width=0.38\textwidth]{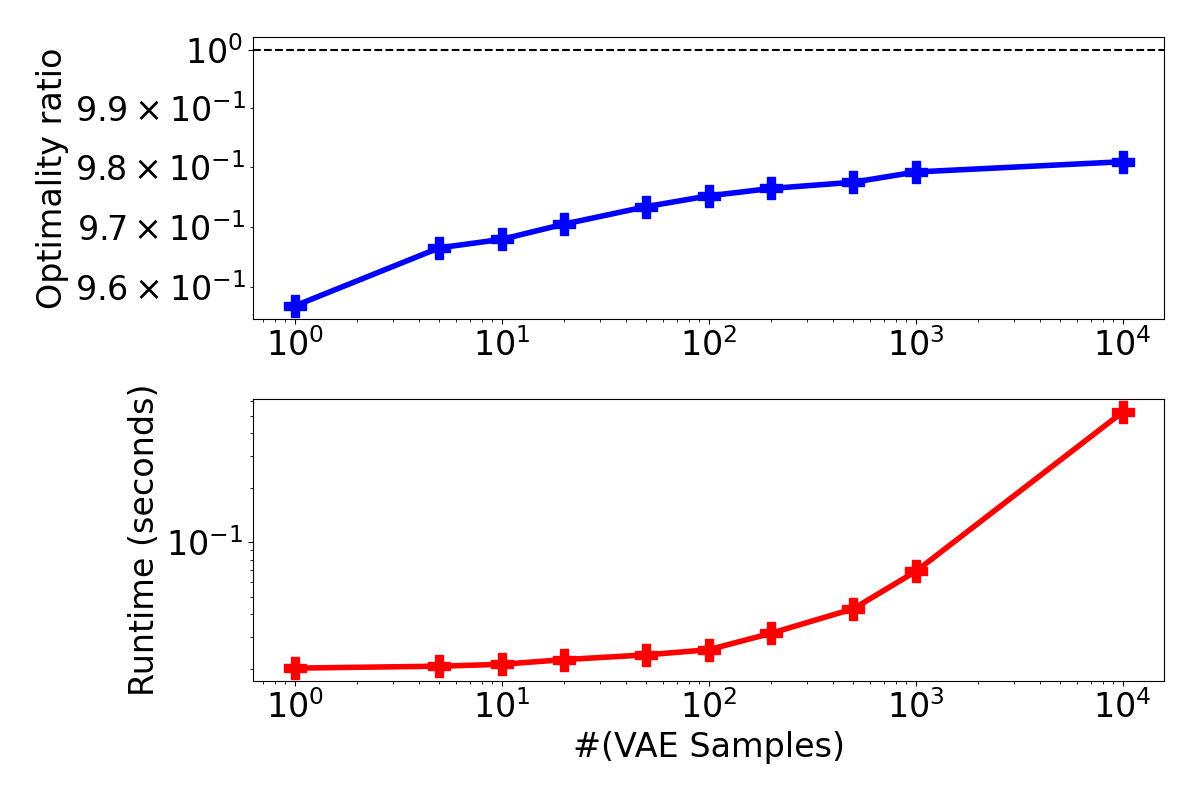}
    \caption{Ablation study on the number of GNN-VAE samples.}
    \label{fig:ablation-on-samples}
\end{figure}
\subsection{Ablation studies}
At the testing phase for the ``Average completion time" dataset, we sample various numbers of samples per graph and evaluate the performance. As shown in Fig.~\ref{fig:ablation-on-samples}, with one sample used, the optimality ratio is 0.96, and as the number of samples increases, the optimality ratio improves and finally converges to 0.98 at the cost of increasing runtime. This shows the advantage of using VAE for assignment prediction, as we can pick the one with the highest performance from multiple candidates. To balance the quality and the runtime, we generate 100 samples per graph in our experiments.

\section{CONCLUSIONS}
\label{sec:conclu}

We propose a Graph Neural Network Variational Autoencoder (GNN-VAE) framework to generate high-quality solutions for a multi-agent coordination problem. Treating coordination as a graph optimization problem, we design GNN-VAE to learn assignments in a semi-supervised manner from the optimal solutions. Our GNN-VAE has been proven to generate feasible solutions that satisfy the acyclic and density constraints inherent in coordination problems. Trained in small-scale problems, our method shows great generalizability and scalability in large-scale problems, generating near-optimal solutions 20 times faster than the oracle and achieving better quality-efficiency trade-offs than other baselines. However, some limitations remain: our approach relies on ground truth data and thus cannot adapt to flexible cost functions after training. Besides, we assume a fully observable environment without uncontrollable agents (e.g., pedestrians). We aim to address these in future work.

\addtolength{\textheight}{-6cm}

\bibliographystyle{plain}
\bibliography{z7_reference}
\end{document}